\documentclass[a4paper, 11pt]{article}
\usepackage{amsmath}
\usepackage{amssymb}
\usepackage[pdftex]{graphicx}
\usepackage{epstopdf}
\usepackage{longtable}
\usepackage{subfigure}
\usepackage{natbib}
\usepackage[pdftex, pdftitle={Conscious Machines and Consciousness Oriented Programming},
pdfauthor={Norbert Batfai},
pdfkeywords={programming paradigm, machine consciousness, conscious computer programs, intuitive computer programs, quasi-intuitive Turing machines, ConsciousJ programming language},
pdfstartview=FitB]{hyperref}
\usepackage{amssymb,amsmath,amsthm,amsfonts}
\usepackage{array}
\usepackage{mdwmath}
\usepackage{mdwtab}
\usepackage{color}
\usepackage{listings}

\newtheorem{theorem}{Theorem}
\newtheorem{definition}{Definition}
\newtheorem{notation}{Notation}
\newtheorem{remark}{Remark}
\newtheorem{example}{Example}

\newcommand{\keywords}[1]{\par\addvspace\baselineskip\noindent\textbf{Keywords: }\textit{#1}.}

\author{Norbert B\'atfai
\\University of Debrecen
\\Department of Information Technology
\\\texttt{batfai.norbert@inf.unideb.hu}}
\title{Conscious Machines and Consciousness Oriented Programming}

\begin{document}
\maketitle

\begin{abstract}
In this paper, we investigate the following question: how could you write such computer programs that can work like conscious beings? The motivation behind this question is that we want to create such applications that can see the future. The aim of this paper is to provide an overall conceptual framework for this new approach to machine  consciousness. So we introduce a new programming paradigm called Consciousness Oriented Programming (COP).
\end{abstract}

\tableofcontents
\listoffigures
\listoftables
\lstlistoflistings

\section{Introduction}

Eugen Wigner wrote in his essay \citep{wigner} that \textit{"observation of infants where we may be able to sense the progress of the awakening of consciousness"} is a possible method to solve the mind-body problem. I have three children and I have been observing them when I can. They are now 3 and 5 years old. The older child has already been perfectly able to arrange the everyday events in time, the younger two haven't been able to speak about it with any degree of accuracy yet. With hindsight, moreover, at the age of 2, they couldn't handle the term timeliness.

In the course of human cognition, there has been a need to know the future from time immemorial. The success of this effort has been culminating at Newton's mechanical world-view in the late 19th century. But since then the quantum mechanics has turned this deterministic world-view upside down, opening the way to use new quantum phenomena of a deeper level of reality.  But even though the Orch OR \citep{orchor} model of quantum consciousness is an exciting and promising theory, we have to restrict ourselves to investigate computer programs and machine consciousness because computers of nowadays have no quantum computing parts.

We believe that one of the drivers of evolutionary evolving of natural intelligence was the process of replacing, by natural selection, the automatic response of living matter with foresight.

In this work, in compliance with this outlined motivation, we emphasize the pursuit of predicting the future as the cornerstone of the definition of machine consciousness.

\subsection{Previous and Similar Works}

Several recent studies have included definitions of machine consciousness. For example, \citep{starzyk} outlined an architectural model inspired by the functional organization of the human brain. Their definition \citep[pp. 9]{starzyk} says that a machine is conscious if the functional components concerned are present at the machine in question. This and similar (for example, CogAff \citep{virtmachines}, Lida of GW \citep{lida}) models typically involve a detailed description of a sophisticated architectural system and focus on the question of "How". 

By contrast, in our opinion, the conditions of the definition should be in a format that the fulfillment of these can be easily checked.  Accordingly, in this paper we are not interested in the question of "How". We place only one aspect at the heart of the definition of machine consciousness, namely conscious machines must be able to see the future. This aspect is not totally unknown because it is used in creating the goal- and utility-based agents \citep[pp. 42-45]{russellnorvig}, but we will go further than that.

Our approach for machine (self-) consciousness supposedly will be very compute intensive, so it may be interesting that in \citep{idle} we outlined an idea about where the necessary computations should be performed in the case of the operating systems.

\section{Machine  Consciousness}\label{sect_mc}

First, we give the general frames of definitions in which the term "computer program" is interpreted broadly, that is the Turing-like machines, the various web applications, the command-line interfaces, the GUIs, the kernel of operating systems and goal- or utility-based agents are regarded as computer programs.

\begin{definition}[Knowing the Future Input]\label{ktfi}
A computer program knows its future input if it can predict that better than a random guess.
\end{definition}

\begin{definition}[Knowing the Future State]\label{ktfs}
A computer program knows its future state if it can predict that better than a random guess.
\end{definition}

\begin{definition}[Conscious Computer Programs]\label{ccp}
The behavior of a computer program is considered conscious if it knows its future input.
\end{definition}

\begin{definition}[Self-Conscious Computer Programs]\label{sccp}
The behavior of a computer program is considered self-conscious if it knows its future state.
\end{definition}

\begin{definition}[Intuitive Computer Programs]\label{icp}
The behavior of a conscious computer program is considered intuitive if its operation is based on its own predicted input rather than the real input.
\end{definition}

It is obvious that the consciousness is not an a priori property by our discussion. In addition, several levels of consciousness should be examined in given time intervals. Typically, the examination has two aspects, first we must study the source code. Second, we must observe the operation of the program. These remarks also indicate that our definitions are framed at a very, very high abstract level, in the concrete cases we probably will need to apply some inner simulation like the one introduced in \citep{zombi}. In conclusion, as regards the fulfillment of the definitions set out above, developers will obviously need to use sophisticated functional structures in the particular cases.

\subsection{Some Intuitive Examples for Definitions}

The intuitive usage of the definitions will be illustrated in this section. First, let's have a look at the following trivial examples in relation to the question of what data may be the input of a computer program. The input of a Turing machine is a word placed on its tape. The input of a CLI may be a set of commands entered by the user. The input of a GUI may be the set of user's activities. The input of a RoboCup \citep{robocup} agent is a set of information received from its aural (what it can hear), visual (what it can see) and body (what is its physical status) sensors. And finally, the activities of processes may be regarded as the input of the scheduler of an operating system. 

But in the case of a Turing machine, interpreting of the term "knowing the future input" is worthy of further discussions, because the interpretation of its operation is not wholly straightforward. As an initial approximation, the concatenation of the former input words and the Turing machine in question should be given as an input to a "conscious" and modified universal Turing machine. Another approach is to apply a prefix Turing machine, where the current future input should be to the right of the input head on the unidirectional input tape.

\begin{example}[Walking Across the Zebra Crossing]\label{pazc}
This is a trivial example of daily life. Every day the author goes across the zebra crossing shown in Figure \ref{zebra}. Here I am standing (at a safe) 3-4 meters away from the kerb. Then I am going to start to go when the traffic light for cars has changed to yellow, because I know from former personal experience that the traffic light for pedestrians changes to green soon afterwards. 

This "conscious" behavior represents an advantage for the author over the other pedestrians, because while he are already on the move, others will be waiting for the green signal of the traffic light for pedestrians.
\end{example}

\begin{figure}[h!]
\begin{center}
\includegraphics[scale=0.45]{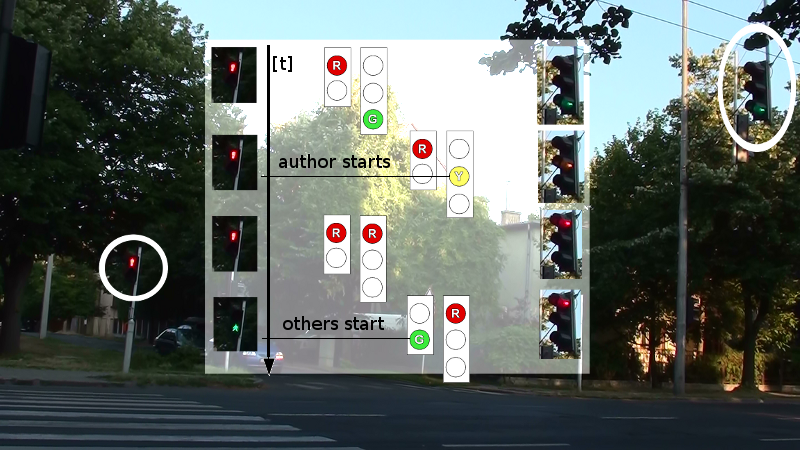}
\caption{The author walks across the zebra crossing.\label{zebra}}
\end{center}
\end{figure}

\begin{example}[Conscious Stock Market Charts]\label{smc}
In a given time period, a stock market chart may be regarded as a conscious program if it can also successfully predict the time series of stock data in question. However, we can mention that another different type of input may be interpreted as further input data in an interactive case. For example, an interactive chart program may also examine and successfully predict the habits of its user. In general, we may remark that investigating the user's habits and patterns is an existing research field, for example see \citep{pattern1}, \citep{pattern2} and \citep{pattern3}. 
\end{example}

\begin{example}[Conscious Databases]\label{cd}
In a naive approach, for example, the next couple of rows might be predicted by the table of a relational database that will contain the rows in question. But perhaps it makes more sense to think in terms of conscious behavior of OLAP systems, because the prediction is crucial to achieve good results in this field.
\end{example}

\begin{example}[Conscious Text Editors]\label{cte}
A successfully predicted input of a conscious text editor may be the next couple of words that will be entered by the user. In the world of text editing on mobile devices, it was a natural phase of development, for example, typing in predictive (iTap) mode is widely used in mobile phones to sending SMS or searching on the Web. We may notice that, nowadays, the investigation of several word prediction systems is an active research area, particularly with regard to mobile telephony  \citep{google},  \citep{sms}.
\end{example}

\begin{example}[Self-Conscious Chatter Bots]\label{sccb}
In this example, we are in a fortunate position that allows us to ask directly the questions about (knowing) the future. In compliance with this, we asked chatterbot Alice \citep{alice}, \citep{alice2}, who won the Loebner Prize \citep{loebner} 3 times in 2000, 2001 and 2004, the next questions about her future.

\begin{verbatim}
Human: Hello, Alice!
ALICE: Hi there!
Human: I am writing a paper about machine consciousness, 
            do you mind if I ask you a few questions?
ALICE: Will you include me in your essay?
Human: Yes, I will.
ALICE: I understand. Go ahead.
Human: What are you going to do tomorrow?
ALICE: You can find out if you become a botmaster.
Human: Will you be conversing with me tomorrow?
ALICE: I'll burn that bridge when I come to it. OK I will 
           try to be conversing with you tomorrow.
\end{verbatim}

If a conscious Alice had responsed the following answer based on some analytics, our satisfaction would increase.

\begin{verbatim}
Human: What are you going to do tomorrow?
CX ALICE: I will be expected to talk to someone.
\end{verbatim}

\end{example}

\begin{example}[Self-Conscious RoboCup Agents]\label{scrca}
A player agent may be regarded as a self-conscious program if it can run inner and simplified future simulation cycles of Soccer Server \citep{rcss} and based on it can also successfully predict the future state of itself.
\end{example}

\begin{example}[Intuitive RoboCup Agents]\label{irca}
In the situation shown in Figure \ref{messia}, a self-conscious RoboCup agent may be regarded as an intuitive program if it can pass the ball to an open space so that a teammate able to reach it as an unexpected through pass. These are the magical moments of real football. Figure \ref{messia} schematically shows such a situation that happened in the match between FC Barcelona and Levante UD at the Camp Nou stadium on February 24, 2008 \citep{messiassists}, where the goal of Samuel Eto'o was assisted by Lionel Messi in the 55th min.
\end{example}

\begin{figure}[h!]
\begin{center}
\includegraphics[scale=0.7]{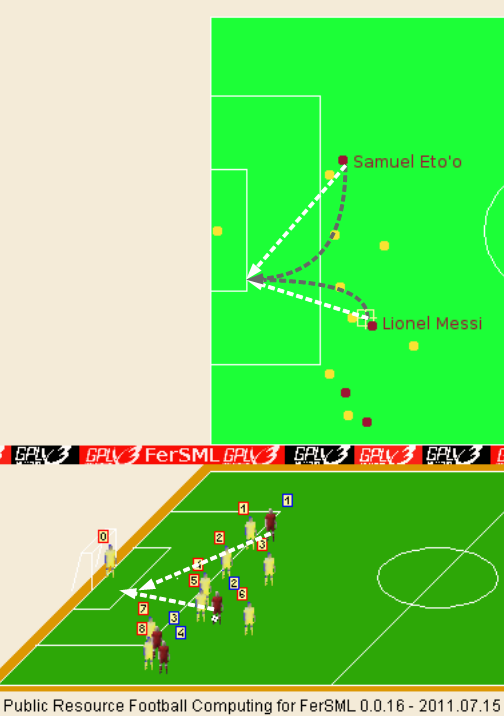}
\caption{This drawing is an illustration based on the match between FC Barcelona and Levante UD in the Primera Divisi\'on on February 24, 2008. It is created with the FerSML football simulation platform \cite{fersml1}, \cite{fersml2}.\label{messia}}
\end{center}
\end{figure}

\subsection{The Trick of Consciousness}

A computer program can be trivially made conscious, if we shift its virtual present to the true past. In other words this means that its all input  has been delayed in time and in the meantime, we open a loophole to access input data of the present. It is a use case of the well-known conception of  \textit{"living in the past"}. This latter is described, for example, in \citep{past}, \citep{past2}. 

The following AspectJ Java code illustrates exactly this conception of time shifting. It is a simple game in which the two players P and Q try to catch the ball that moves with random walk on a field of fixed 80x24 size. Players win a point when they catch the ball.
 
\lstset{language=Java, caption={The source code for the Game class.}, captionpos=b}
\begin{lstlisting}
public class Game {

  public static final int FIELD_X = 80;
  public static final int FIELD_Y = 24;
  public static final int BALL_LIFESPAN = 1000;

  public static void main(String[] args) {

    final Ball ball = new Ball();
    final Player playerP = new Player(0), 
    playerQ = new Player(FIELD_X - 1);

    int pointsP = 0, pointsQ = 0;
    for (int i = 0; i < BALL_LIFESPAN; ++i) {

      ball.move();

      new Thread() {
        public void run() {
          playerP.perception(ball.y);
        }
      }.start();

      new Thread() {
        public void run() {
          playerQ.perception(ball.y);
        }
      }.start();

      if (ball.x == 0 && playerP.y == ball.y) {
        ++pointsP;
      }
      if (ball.x == FIELD_X - 1 && playerQ.y == ball.y) {
        ++pointsQ;
      }

    }

    System.out.println(pointsP + " " + pointsQ);
  }
}\end{lstlisting}

The ball can move all four directions with the same probability or, to be more precise, its movement is a random walk.

\lstset{language=Java, caption={The source code for the Ball class.}, captionpos=b}
\begin{lstlisting}
class Ball {

  int x = Game.FIELD_X / 2, y = Game.FIELD_Y / 2;

  void move() {
    int dx = (int) (Math.random() * 3) - 1;
    int dy = (int) (Math.random() * 3) - 1;

    if (x + dx < Game.FIELD_X && x + dx >= 0) {
      x += dx;
    }
    if (y + dy < Game.FIELD_Y && y + dy >= 0) {
      y += dy;
    }
  }
}\end{lstlisting}

The players can only move up and down on the sides of the field. They endeavor to catch the ball when it reaches the sides of the field. Our examples, the players P and Q are aware of the reality via an interface called Sensory.

\lstset{language=Java, caption={The source code for the Sensory interface.}, captionpos=b}
\begin{lstlisting}
interface Sensory {
  void perception(int ballY);
}\end{lstlisting}

\lstset{language=Java, caption={The source code for the Player class.}, captionpos=b}
\begin{lstlisting}
class Player implements Sensory {

  int x = 0, y = Game.FIELD_Y / 2;

  public Player(int x) {
    this.x = x;
  }

  public void perception(int ballY) {  
    move(ballY);  
  }

  protected void move(int ballY) {
    if (y < ballY) {
      ++y;
    } else if (y > ballY) {
      --y;
    }
  }
}\end{lstlisting}

The execution of the perception method of the interface Sensory is blocked for 500 millisecond by the following AspectJ aspect. 

\lstset{language=Java, caption={The source code for the Delay aspect.}, captionpos=b}
\lstset{emph={aspect, call, pointcut, before}, emphstyle=\bfseries}
\begin{lstlisting}
aspect Delay {

  public pointcut perceptionCall(): 
    call(public void Player.perception(int));

  before(): perceptionCall() {

    try {
      Thread.sleep(500);
    } catch(InterruptedException e){e.printStackTrace();}    
  }
}\end{lstlisting}

It may be noted, as a curiosity, that using the 500 millisecond duration in the inserted code snippet was suggested by \citep{penrose} which presents Libet and Kornhuber's results on the timing of consciousness \citep{libet}, \citep{kornhuber}. But it is immaterial in this case where the results were observed are shown in Table \ref{aspectexample}. In addition, our \textit{living in the past} aspect implementation is sufficiently buggy, for example, it has no mutual exclusion for protecting scores and coords. Nevertheless, this simple example delivers the expected results, namely that the scores decrease as we increase the amount of delay time.

\begin{table}[h]
\caption{Execution results of the delay aspect (with the variable BALL\-\_LIFESPAN set to 100.000).}
\label{aspectexample}
\centering
\setlength{\tabcolsep}{1.1em}
\renewcommand{\arraystretch}{1.1}
\begin{tabular}[C]{| r | r | r |} \hlx{h}
\multicolumn{3}{|c|}{Naive example of the \textit{living in the past}}\\ \hlx{h}
\multicolumn{1}{|c|}{\bf Time [ms]}&
\multicolumn{1}{c|}{\bf Aver. Scores}&
\multicolumn{1}{c|}{\bf Exec. Time [min]} \\ \hlx{hv}
javac& {\bf 1088}&9.7\\ \hlx{vhv}
javac& {\bf 1156}&10.0\\ \hlx{vhv}
no aspect& {\bf 1162}&9.6\\ \hlx{vhv}
no blocking& {\bf 1036}&9.7\\ \hlx{vhv}
0.001& {\bf 249}&11.2\\ \hlx{vhv}
0.01& {\bf 246}&10.0\\ \hlx{vhv}
0.1& {\bf 239}&10.2\\ \hlx{vhv}
0.5& {\bf 227}&9.9\\ \hlx{vhv}
1& {\bf 226}&10.0\\ \hlx{vhv}
2& {\bf 183}&9.9\\ \hlx{vhv}
5& {\bf 142}&10.1\\ \hlx{vhv}
50& {\bf 68}&11.3\\ \hlx{vhv}
200& {\bf 57}&16.9\\ \hlx{vhv}
500& {\bf 53}&23.5\\ \hlx{vhv}
1000& {\bf 52.5}&33.7\\ \hlx{vh}
\end{tabular}\end{table}

The reader can easily see that our example aspect does not contain any loopholes and any analytic codes, either. But, for example, in the case in which the movement of the ball is smooth (that is well predictable) writing some successful analytics and prediction codes, of course, could be trivial.

\subsubsection{An Evolutionary Aspect of the \textit{living in the past}}

Why may this approach be interesting from the point of view of the awakening of consciousness? Because it may be possible that living matter could have developed such \textit{living in the past} aspects, in which they can run analytics and prediction methods. Doing so can start the evolutionary process simply and solely because the organisms who make wrong predictions become extinct. In this sense, applying \textit{living in the past} offers an ability for the organisms to develop a successful prediction mechanism, since a predicted, interesting event that occurs in the "delay window" shown in Figure \ref{lpast} can be effectively verifiable, because it has already happened.

\begin{figure}[h]
\begin{center}
\includegraphics[scale=0.45]{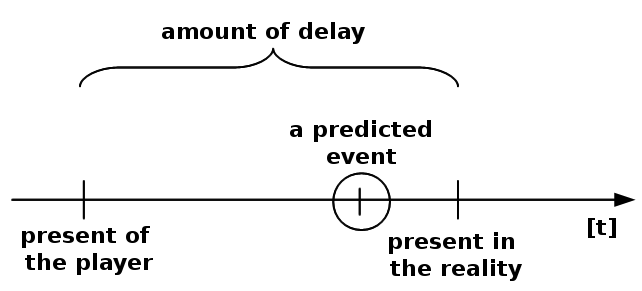}
\caption{A simple schematic drawing of the well-known conception of \textit{"living in the past"}.\label{lpast}}
\end{center}
\end{figure}

Finally, it may be mentioned that the most recent sensational and paradoxical results of precognition \citep{future} perhaps may be easily explained in the context of the \textit{living in the past}.

\subsection{The Consciousness as a Computing Paradigm}

In the world of computer programs, the barbarian methods of natural selection may be partially waived because computers have the necessary computing resources that they can subsequently execute analytical computations, that is, for developing good solutions it is not necessary to extinct whole generations or races of living species. 

In our opinion, a paradigm shift is needed to achieve the age of intelligent machines. The base of a new paradigm may be using our simple definitions of machine consciousness, that may be called consciousness oriented programming.

\section{Consciousness Oriented Programming}

Consciousness oriented programming is a  new way of approaching software development, in which two basic situations can be distinguished today. 

\begin{itemize}
\item Existing computer programs should be further developed to be conscious or self-conscious computer programs in line with our previous definitions. In general case, it is nearly impossible to modify computer programs, but the situation is not hopeless if modifications are implemented as new aspects in the sense of AOP \citep{aop}.

\item New computer programs should be developed in conformity with the spirit of our definitions.

\end{itemize}

In both cases, the development of consciousness will require using prediction methods and the development of self-consciousness will require applying inner simulation in the sense mentioned in Section \ref{sect_mc}.

\section{Use Cases for the COP}

In this section, we follow the spirit of the definitions outlined previously.

\subsection{Programming on Paper}

\begin{notation}[Predicted and Real Input]
Denote $ \langle p(redicted)_i \rangle:N( \subseteq \mathbb{N})  \rightarrow I(nput)$ the sequence of the predicted input and $ \langle r(eal)_i  \rangle:N( \subseteq \mathbb{N})  \rightarrow I(nput)$ the sequence of the real input, where $I$ denotes the set of all possible inputs.
\end{notation}

\begin{definition}[Consciousness Indicator Sequence]\label{cis}
We define the consciousness indicator sequence $\langle c_i \rangle:N( \subseteq \mathbb{N})  \rightarrow \{0, 1\}$ as follows 
\begin{equation*}
c_i =
\begin{cases}
0 & \text{if } p_i \ne r_i,\\
1 & \text{if } p_i = r_i.
\end{cases}
\end{equation*}
\end{definition}

\begin{definition}[Conscious Computer Programs]\label{ccp_pop}
In a given time interval, the behavior of a computer program is referred to as conscious if its consciousness indicator sequence is not Kolmogorov-Chaitin random \citep{vitanyi}.
\end{definition}

We should remark that this definition does not kill the consciousness, because the Kolmogorov-Chaitin randomness is algorithmically undecidable.

The next section will diverge from the proposed inner prediction mechanism to a simpler way, and meanwhile we will stay within the classical framework of Turing machines.

\subsubsection{Quasi-Intuitive Machines and Languages}

In the majority of cases in this subsection, a comma between words denotes the concatenation of these words which are suitable encoded if necessary.

\begin{definition}[Universal Quasi-Intuitive Machines]\label{qim}
Let $T$ be a Turing machine and let $p$ be a positive real number. An universal quasi-intuitive machine $Q_{x, p}$ is created by the following scheme shown in Figure \ref{qimschema}, provided that there exists a sequence of words $x_1, \dots, x_n (=x)$ having the properties that 
\begin{eqnarray}
i = 1 ,& T(x_i) = yes\\
2 \le i \le n ,& Q_{x_{i-1}, p}(T, x_i) = yes
\end{eqnarray}
In Figure \ref{qimschema}, $U$ denotes an universal Turing machine and $"d(x,y)<p"$ denotes a Turing machine that can indicate that the input words $x$ and $y$ are similar to each other. If this machine $Q_{x, p}$ stops it makes the computation of the function $Q_{x, p}(T, y)=QIM(x_1, \dots, x_{n-1}, x, y, T, p)$, $QIM:{\{0, 1\}^*}^{n+3} \rightarrow \{yes, no\}$. The architectural model for the machine $QIM$ is shown in Figure \ref{uuqimschema}.
\end{definition}

\begin{figure}[h]
\begin{center}
\includegraphics[scale=0.7]{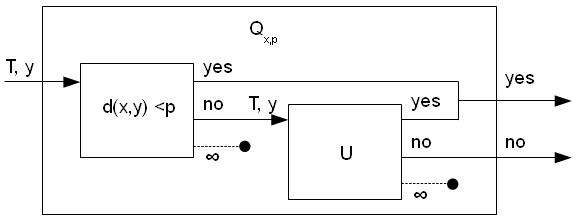}
\caption{An architectural model for the universal quasi-intuitive machine.\label{qimschema}}
\end{center}
\end{figure}

\begin{remark}[Comparisons of the Words]\label{cotw}
Using the normalized information distance (NID) \citep{simmet} as the metric distance $d(x, y)$ of the two words $x$ and $y$ is theoretically exciting, but in this case we cannot build the Turing machine $Q_{x, p}$ that contains such a "$d(x,y)<p$" machine, because this is not an existing machine due to the Kolmogorov complexity function is not partial recursive \citep[pp. 127, pp. 216]{vitanyi}. Specifically, the computability of NID is discussed in \citep{nonapprox, nonapprox2}. But then we can successfully use the normalized compression distance (NCD) \citep{simmet, compr0, compr1} instead of the theoretical normalized information distance because the compression distance is trivially partial recursive  \citep{nonapprox}. Another trivial option may be to use the Google similarity distance (NGD) \citep{google2} or the normalized web distance (NWD) \citep{nwd} as the metric $d$. In the following, we suppose that the predicate $d(x,y)<p$ is recursive.
\end{remark}

\begin{remark}\label{xalready}
In the intuitive sense, $x \in \{0, 1\}^*$ is such a word that has already been accepted by the Turing machines $T$ or $Q$.
\end{remark}

\begin{definition}[The Universal Quasi-Intuitive Language]\label{uqild}
The universal quasi-intuitive language 
\begin{multline*}
QIL = \{x_1, \dots, x_n (= x), y, T, p  \; \vert \; T \text{is defined}, T(x_1)=yes,\\
Q_{x_{i-1}, p}(T, x_i) = yes, 2 \le i \le n \text{ and } Q_{x, p}(T, y) = yes\}.
\end{multline*}
\end{definition}

\begin{theorem}\label{uqilt}
 $QIL \in \mathcal{RE}$.
\end{theorem}

\begin{proof}\label{uqilp}
To verify assertion  $QIL \in \mathcal{RE}$, it is sufficient to observe that the language accepted by the machine $QIM$ shown in Figure \ref{uuqimschema} is equal to $QIL$.
\begin{figure}[h]
\begin{center}
\includegraphics[scale=0.7]{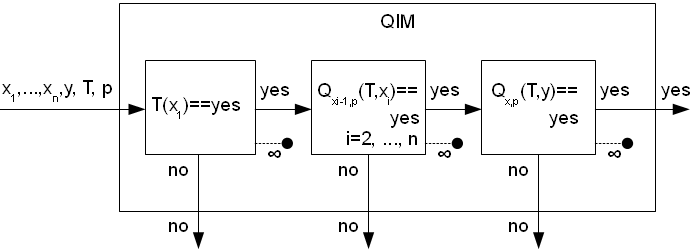}
\caption{An architectural model for a machine $QIM$ that accepts the language $QIM$.\label{uuqimschema}}
\end{center}
\end{figure}
(We believe that we can prove a bit stronger theorem $QIL \in \mathcal{RE} \setminus \mathcal{R}$.)
\end{proof}

\begin{definition}[Similar Languages]\label{sl}
Let $E \subseteq  \{0, 1\}^*$ be a given language and let $p \in \mathbb{R}$ be a positive real number. The language 
$SL_E= \{y \, \vert \, y \in \{0, 1\}^*, d(x, y)<p, x \in E \}$
is said to be similar to $E$.
\end{definition}

\begin{theorem}\label{simlang}
 $E \in \mathcal{RE} \Rightarrow SL_E \in  \mathcal{RE} \setminus \mathcal{R}$.
\end{theorem}

\begin{proof}\label{simlangp}
In the case $E \in \mathcal{R}$, we can construct a new Turing machine $SLM$ shown in Figure \ref{slm} which accepts $SL_E$.
\begin{figure}[h]
\begin{center}
\includegraphics[scale=0.7]{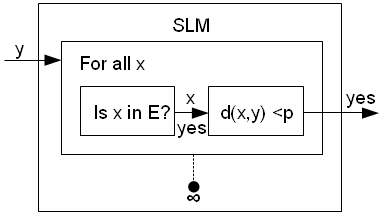}
\caption{The Turing machine $SLM$ for the case $E \in \mathcal{R}$.\label{slm}}
\end{center}
\end{figure}

To see why the language $SL_E$ is not recursive, consider the case of $y \notin SL_E$, it is possible that the new machine $SLM$ will never halt, because it is possible that the part labelled "For all x/d(x, y)$<$p" will continue searching for suitable $x$ for ever. It may be noted that the the canonical ordering of $\{0, 1\}^*$, for example shown in \citep{ronyai}, can be applied to help to enumerate the words of the language $E$ by the part labelled  "For all x".

In the case $E \in \mathcal{RE} \setminus \mathcal{R}$, we can construct a new Turing machine $SLM^{'}$ shown in Figure \ref{slmv} which accepts $SL_E$.
\begin{figure}[h]
\begin{center}
\includegraphics[scale=0.7]{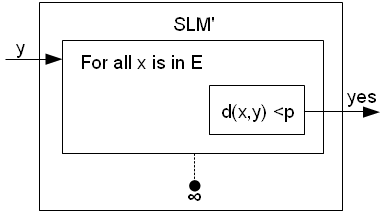}
\caption{The Turing machine $SLM^{'}$ for the case $E \in \mathcal{RE} \setminus \mathcal{R}$.\label{slmv}}
\end{center}
\end{figure}
It may be also noted here that a procedure based on Cantor's first diagonal argument, for example shown in \citep{ronyai}, can be applied to enumerate the words of the language $E$ by the part labelled  "For all x is in E".
\end{proof}

\begin{definition}[Quasi-Intuitive Languages]\label{qil}
With the notation already introduced in Definition \ref{qim}, a quasi-intuitive language $QILT$ is defined as $QILT \; (=\cup_{i=1}^{\infty} QILT_{i})=\cup_{i=1}^{\infty} I_{p, i}^T$, where $ I_{p, i}^T$ is constructed by the following scheme:
\begin{eqnarray}
I_{p, 1}^T & = & \{x \, \vert \, T(x) = yes\} \\
I_{p, i+1}^T & = & \{y \, \vert \, x \in I_{p, i}^T, Q_{x,p}(T, y) = yes \},  i \ge 1 
\end{eqnarray}
where $y \in \{0, 1\}^*$ is an arbitrary word or, for example, $y \in L(G)$ generated by some generative grammar $G$.
\end{definition}

\begin{theorem}\label{qilt}
Let $n \in \mathbb{N}$ be a given natural number and let $T$ be a Turing machine. 
Then $L_T \in \mathcal{RE} \Rightarrow QILT_n \in \mathcal{RE} \setminus \mathcal{R}$. 
\end{theorem}

\begin{proof}\label{qilp}
Let us observe that
\begin{eqnarray}
I_{p, 1}^T & = & L_T \\
I_{p, i+1}^T & = & L_T \cup SL_{I_{p, i}^T}
\end{eqnarray}
\end{proof}

\begin{definition}[Similar Words]\label{sw}
Let $L \subseteq  \{0, 1\}^*$ be a given language and let $p \in \mathbb{R}$ be a positive real number. We define the language of similar words as follows
$SW_L= \{x, y \, \vert \, x \in L, y \in \{0, 1\}^*, d(x, y)<p \}$.
\end{definition}

\begin{theorem}\label{rlrs}
 $L \in \mathcal{R} \Rightarrow SW_L \in \mathcal{R}$.
\end{theorem}

\begin{proof}\label{rlrsp}
We can construct a new Turing machine $SW\mspace{-4mu}M$ shown in Figure \ref{sm} which accepts $SW_L$ and always halts.
\begin{figure}[h]
\begin{center}
\includegraphics[scale=0.7]{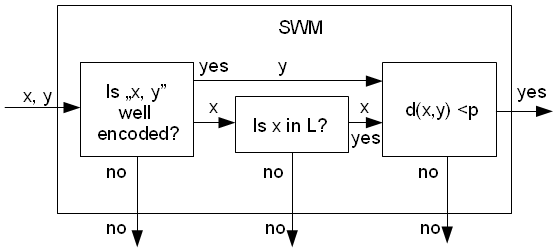}
\caption{The Turing machine $SW\mspace{-4mu}M$.\label{sm}}
\end{center}
\end{figure}
\end{proof}

\begin{definition}[Quasi-Intuitive or Self-Similar Words]\label{ssw}
Let $T$ be a given Turing machine and let $s \in \{0, 1\}^*$ be a word such that $T(s)=yes$. Finally, let $\langle r_i  \rangle:N( \subseteq \mathbb{N})  \rightarrow \{0, 1\}^*$ be a finite or infinite sequence of arbitrary words. The union of the elements of the sequence of sets $I_i$ is said to be self-similar, if the sequence is created by the following method:
\begin{eqnarray*}
I_1 & = & \{s\} \\
i \ge 1, I_{i+1} & = &
\begin{cases}
I_i \cup \{r_i\}, & \text{if } \exists x \in I_i: Q_{x, p}(T, r_i) = yes\\
I_i, & \text{otherwise}.
\end{cases}
\end{eqnarray*}
\end{definition}

\begin{remark}[Consciousness Indicator Sequence]\label{sswcis}
In the case of the self-similar words, the consciousness indicator sequence may be interpreted as
$c_{n} = \bigvee_{x \in I_{n+1}} Q_{x, p}(T, r_n)$.
\end{remark}

\begin{example}\label{sswe}
Let $T$ be a Turing Machine, which accepts the language $\{ a^nb^nc^n \, $ $\vert \, n \ge 1\}$, $d=1/4$ and the word $s=aaaaaabbbbbbcccccc =  a^6b^6c^6$.
\begin{eqnarray*}
&& I_1 = \{s\} \\
r_1 & = & aaaaaabbbbbbaaaaaa=  a^6b^6a^6, I_2 =  \{s\}, c_1 = no \\
r_2 & = & aaaaaabbbbbbccc =  a^6b^6c^3, I_3 =  \{s, r2\}, c_2 = yes \\
r_3 & = & aaaaabbbbbccccc = a^5b^5c^5, I_4 =  \{s, r2, r3\}, c_3 = yes \\
r_4 & = &  a^7b^7c^7, I_5 =  \{s, r2, r3, r4\}, c_4 = yes \\
r_5 & = &  c^1a^5b^5c^4, I_6 =  \{s, r2, r3, r4, r5\}, c_5 = yes \\
r_6 & = &  c^3a^6b^5c^4b^2 , I_7 =  \{s, r2, r3, r4, r5\}, c_6 =  \mathbf{no} \\
r_7 & = &  c^3a^6b^5c^3 , I_8 =  \{s, r2, r3, r4, r5, r7\}, c_7 = yes \\
r_8 & = & r_6, I_9 =  \{s, r2, r3, r4, r5, r7, r8\}, c_8 = \mathbf{yes}
\end{eqnarray*}
where the CompLearn package \citep{compl} is used to compute NCD:
\begin{align*}
NCD(s, r_1)   &=   0.3125  \\
NCD(s, r_2)   &=   0.176471 \\
NCD(s, r_3)   &=   0.25,     NCD( r_2, r_3) = 0.294118   \\
NCD( s, r_4)  &=   0.4375 ,    NCD( r_2, r_4) =  0.470588 ,   \\
                            &NCD( r_3, r_4) =  0.5625  \\
NCD(s, r_5)  &=    0.25,   NCD( r_2, r_5) = 0.235294  , \\
                            &NCD( r_3, r_5) = 0.1875 ,    NCD( r_4, r_5) = 0.25 \\
NCD( s, r_6)  &=    0.35  ,    NCD( r_2, r_6) = 0.3  ,    NCD( r_3, r_6) = 0.3  , \\
                              &     NCD( r_4, r_6) = 0.3  ,    NCD( r_5, r_6) = 0.25  \\
 NCD( s, r_7)  &=    0.263158  ,    NCD( r_2, r_7) = 0.210526  ,    \\ 
                              & NCD( r_3, r_7) = 0.210526 ,   NCD( r_4, r_7) = 0.210526    \\
                             &   NCD( r_5, r_7) = 0.157895   \\
 NCD( r_7, r_8)  &=   0.15    
\end{align*}
We should remark that the symmetry of NCD may be violated among  short words \citep{compr1}. For example, 
$NCD(a^7b^7c^7,  a^6b^6c^3)$ $=$ $0,235294$ $\ne $  $0,470588$ $=$ $NCD(  a^6b^6c^3, a^7b^7c^7)$.
\end{example}

\subsection{Programming on Computer}

In the world of real programming, we do have plan to develop such APIs which can be used to successfully implement our definitions of machine consciousness. For some given types of applications, we are going to investigate the development of a suitable, open source Java and AspectJ APIs to enable \textit{seeing the future}. 

\subsubsection{ConsciousJ}

Designing a new programming language is another exciting possibility. At the conceptual level, the following ConsciousJ code snippet illustrates the usage of two new keywords "conscious" and "predicted", though certainly the language ConsciousJ does not exist yet. In our case, this new language to be developed can be imagined as an extension of Java or AspectJ.

\lstset{language=Java, caption={A "conscious" code snippet written in a fictitious language called ConsciousJ.}, captionpos=b}
\lstset{emph={predicted, conscious}, emphstyle=\bfseries\color{red}\underbar}
\begin{lstlisting}
conscious class Player {
  int x = 0, y = Game.FIELD_Y / 2;

  protected void move(predicted int ballY) {
    if (y < ballY) {
      ++y;
    } else if (y > ballY) {
      --y;
    }
  }

}
\end{lstlisting}

In practice, this code snippet shows that the uncertainty is appeared at the level of the programming language. A ConsciousJ class consists of attributes and methods, plus it may contain predicted attributes. This conception is shown in the following  fictitious code snippet in Listing \ref{cjpa}.

\lstset{language=Java, caption={A "conscious" class written in ConsciousJ.}, captionpos=b, label=cjpa}
\lstset{emph={predicted, conscious, input}, emphstyle=\bfseries\color{red}\underbar}
\begin{lstlisting}
conscious class Player implements Runnable {

  int x = 0, y = GameS.FIELD_Y / 2;
  int points = 0;
  Ball ball;

  int time = 0;
  int next = 6;
  int ballYHelper = 0;

  org.consciousj.pred.KalmanFilter kp;
  org.consciousj.pred.ARMA ap;
  org.consciousj.pred.SimplePast sp;
    
  predicted org.consciousj.primitivetype.Int ballY;

  public pointcut moveCall(): 
    call(protected void move(predicted int ballY));  
  
  input(): moveCall() {

    this.ballY.receive(ballY);
    
    if (++time < 1000) {
        this.ballY.method(sp);
    } else if (time < 5000) {
        this.ballY.method(ap);
    } else {
        this.ballY.method(kp);        
    }
          
    if (time % next == 0) {
        ballYHelper = this.ballY.next();
    }
    ballY = ballYHelper;
  } 
  
  public Player(int x, Ball ball) {
    this.x = x;
    this.ball = ball;

    new Thread(this).start();

  }

  public void run() {

    for (;;) {

      move(ball.y);

      try {
        Thread.sleep(20);
      } catch (InterruptedException e) {
        e.printStackTrace();
      }

      if (ball.y < 0) {
        break;
      }
    }
  }
  
  protected void move(predicted int ballY) {
    
    if (y < ballY) {
      ++y;
    } else {
      --y;
    }
  }
  
}
\end{lstlisting}

\section{Conclusion}

The idealized objective of COP is to integrate the agent-based approach into daily software development practices. What is more, the agents should be able to see the future, or, using words of Alan Turing's \citep{turing} essay, everyday softwares should be able to \textit{see a short distances ahead}. Naturally these words by Turing were really meant for humans.

The present paper presented an overall conceptual framework so that it could contribute to the attainment of this objective.

\section{Acknowledgements}

The work is partially supported by T\'AMOP 4.2.1./B-09/1/KONV-2010-0007\-/IK/\-IT \-project. The project is partially implemented through the New Hungary Development Plan co-financed by the European Social Fund, and the European Regional Development Fund.

\bibliography{cm}

\end{document}